%% file: main_arxiv.tex
\documentclass[letterpaper,11pt]{article}
\usepackage[margin=1in]{geometry}

\input{defs}
\usepackage{enumitem}
\usepackage{graphicx,subcaption,xcolor}
\usepackage{url}

\providecommand{\E}{\mathbb{E}}
\providecommand{\Pp}{\mathbb{P}}
\providecommand{\R}{\mathbb{R}}
\providecommand{\A}{\mathcal{A}}
\providecommand{\X}{\mathcal{X}}
\providecommand{\Fcal}{\mathcal{F}}
\providecommand{\Gcal}{\mathcal{G}}
\providecommand{\Ical}{\mathcal{I}}
\providecommand{\Tcal}{\mathcal{T}}
\providecommand{\Zcal}{\mathcal{Z}}
\providecommand{\dd}{\,\mathrm{d}}
\providecommand{\eps}{\varepsilon}
\providecommand{\ind}{\mathbbm{1}}
\providecommand{\Var}{\operatorname{Var}}
\providecommand{\norm}[1]{\left\lVert #1 \right\rVert}
\newcommand{\TCAB}{\texttt{TCAB}}
\newcommand{\TCABstar}{\texttt{TCAB.STAR}}
\newcommand{\TCABmst}{\texttt{TCAB.MST}}

\newtheorem{proposition}{Proposition}
\numberwithin{theorem}{section}
\numberwithin{proposition}{section}
\numberwithin{lemma}{section}
\numberwithin{corollary}{section}
\numberwithin{definition}{section}
\numberwithin{assumption}{section}

\begin{document}

\title{Fast A/B/n Testing: Exact Multi-Policy Comparison via Tree-Coupled Feedback Sharing}

\author{Yuxiao Wen\thanks{Yuxiao Wen is with the Courant Institute of Mathematical Sciences, New York University, email: \url{yuxiaowen@nyu.edu}.}}

\maketitle

\begin{abstract}
Online platforms increasingly compare many adaptive decision policies---ranking systems, recommendation algorithms, pricing rules, and language-model agents---while each reward-bearing interaction can be costly or risky.  A direct A/B/n design gives each of $J\ge 2$ policies its own horizon-$T$ trajectory and therefore uses $JT$ outcomes.  We introduce Tree-Coupled A/B Testing (\TCAB), an exact feedback-sharing design for arbitrary history-dependent contextual-bandit policies.  At each round, a predictable tree connects the current policy histories; every parent--child context--action law is maximally coupled, and one reward is shared within each component of matched tree edges.  Every policy retains exactly its standalone finite-horizon trajectory law, even though the policies are deliberately dependent.  If $D_{e,t}$ records a mismatch on tree edge $e$ at round $t$, the number of reward queries satisfies the pathwise identity $N(T)=T+\sum_{t,e}D_{e,t}$ and hence equals $T$ plus cumulative tree-edge total variation in expectation.  This cost is conditionally optimal among exact edge-local designs on the selected tree, and a current-round minimum-spanning tree is myopically optimal among tree designs.  For fixed $J$, sublinear pseudo-regret of every policy and almost-sure uniqueness of the oracle action imply $\E[N(T)]=T+o(T)$, versus $JT$ for independent runs.  We also obtain finite-sample variance bounds for pairwise policy contrasts.  Experiments on reward-model evaluation, multiple-choice language-model evaluation, and adaptive search policies demonstrate substantial improvements in the cost--precision frontier.
\end{abstract}


\section{Introduction}

Online controlled experiments are production infrastructure at major digital platforms.  Google, Microsoft/Bing, and LinkedIn have described systems that support overlapping experiments and continuous product iteration \citep{tang2010overlapping,kohavi2013large,xu2015infrastructure}.  A cross-industry summit involving thirteen organizations reported that the participating organizations had collectively tested more than one hundred thousand treatment variants in the preceding year \citep{gupta2019challenges}.  The same pressure now appears in machine-learning evaluation: product teams compare many recommender-system configurations, ranking policies, prompting strategies, and model checkpoints, while human or online feedback remains expensive.  For example, the original Chatbot Arena study accumulated more than 240,000 human preference votes to compare language models \citep{chiang2024chatbot}.

A useful distinction is between reusing \emph{traffic} and reusing \emph{feedback}.  Google's overlapping-experiment infrastructure, for example, allows compatible experiments in different layers to share the same underlying traffic while preserving the randomization needed within each experiment \citep{tang2010overlapping}.  Our question is complementary and operates inside a multi-policy comparison: when several candidate policies would produce the same contextual decision, can one realized outcome be used by several policies without changing the finite-horizon law of any of them?  This distinction matters because traffic multiplexing alone does not remove duplicated reward-bearing interactions among the alternatives in a single policy-comparison problem.

The classical A/B experiment is designed for static treatments.  Modern systems instead often compare \emph{policies}: a policy observes the current query or user context, chooses an action, receives only the selected action's outcome, and may update all later decisions from its adaptive history.  To compare $J\ge 2$ candidate policies over the same target horizon $T$, the direct A/B/n design runs $J$ independent trajectories and uses $JT$ reward-bearing interactions.  This cost is operationally important when an experimental action can reduce user experience, consume expert labels, invoke a costly model, or delay a deployment decision.

There is nevertheless a large source of redundancy.  Candidate policies generated by neighboring hyperparameters, model checkpoints, or business rules frequently make the same decision on the same context.  When two policies realize the same complete context--action pair, they require the same conditional reward law and can therefore receive the same physical outcome.  Purposefully coupling their trajectories can remove duplicated reward noise and save an interaction.  The challenge is to do so without changing either policy's adaptive trajectory distribution.

Contextuality makes this challenge fundamentally different from replaying an arm label.  Even under i.i.d. contexts, the contexts attached to the observations of a selected action depend on the policy's selection rule and history.  Moreover, in applications such as ranked slates, auctions, and recommendation, the outcome of choosing one item may depend on the entire displayed slate or user state, not only on the chosen item's local feature.  We therefore model a full-context reward kernel $Q_a(\cdot\mid x)$ and couple the policies' complete one-step \emph{context--action laws}.  A maximal coupling makes two complete pairs equal with the largest probability permitted by their total-variation distance; the shared branch uses one reward and the residual branch opens a fresh query.

Moving from two policies to $J\ge2$ creates a second obstruction.  Pairwise maximal couplings need not be jointly compatible: in general there is no single coupling that maximizes the equality probability of every pair simultaneously \citep{angel2019pairwise}.  We resolve this incompatibility with a policy tree.  Pairwise couplings prescribed on the $J-1$ edges of an acyclic graph can always be glued into one joint law.  A broken edge starts a new reward lineage; matched edges transmit the same context, action, and reward.  Consequently, the number of reward queries is exactly one plus the number of broken edges in each round.

Our algorithm, Tree-Coupled A/B Testing (\TCAB), is round-synchronous.  At the beginning of round $t$, the experiment may select any tree measurable with respect to the histories available through round $t-1$.  It then samples the root, traverses the tree by depth, queries one outcome per matched-edge component, and updates all policies simultaneously.  This outer-loop-over-time formulation directly permits predictable, time-adaptive trees.  It also exposes practical parallelism: all children at the same depth can be coupled concurrently once their parents are available, and the reward queries for distinct components can be issued concurrently.  Adaptive policies remain sequential across rounds, as they must, while nonadaptive policies can additionally parallelize across rounds.

\begin{figure}[h]
    \centering
    \begin{minipage}[c]{0.65\textwidth}
        \centering
        \begin{subfigure}{0.48\linewidth}
            \centering
            \includegraphics[width=\linewidth]{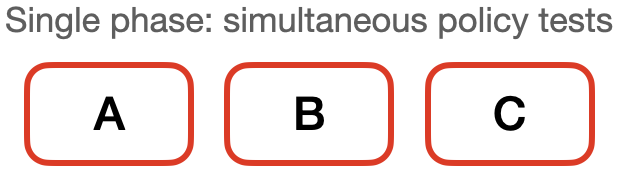}
            \caption{Standard A/B/C test}
            \label{fig:abc_test}
        \end{subfigure}
        \hfill 
        \begin{subfigure}{0.48\linewidth}
            \centering
            \includegraphics[width=\linewidth]{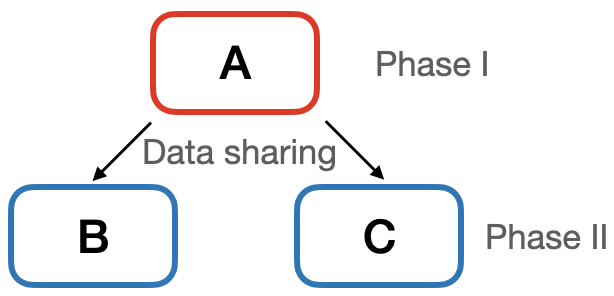}
            \caption{Tree-based with data sharing}
            \label{fig:tree_sharing}
        \end{subfigure}
    \end{minipage}
    \hfill 
    \begin{minipage}[c]{0.30\textwidth}
        \caption{In standard A/B tests, policies are run independently and simultaneously. In \texttt{TCAB}, the policies are run in phases per round, which are determined by the tree, and reuse data from previous phases.}
        \label{fig:tree_example}
    \end{minipage}
\end{figure}

The resulting guarantees are finite-horizon and model-free.  Every policy trajectory has exactly its standalone law, so all values and policy contrasts are unbiased.  The expected query cost is $T$ plus cumulative total variation over the selected tree edges.  This identity yields a precise design principle: a baseline-centered star is simple and makes every baseline comparison maximal, while an oracle minimum-spanning tree minimizes the \emph{current-round} cost among edge-local tree designs.  We emphasize that this is a myopic tree-optimality statement, not a claim of global optimality among arbitrary multi-marginal couplings or over the full adaptive horizon.

The method is most useful when candidate policies are related rather than arbitrary.  Examples include neighboring model checkpoints, nearby hyperparameter settings, alternative ranking or recommendation rules that agree on most users, and learning algorithms that increasingly concentrate on the same good actions.  In the first regime, the relevant tree-edge total-variation distances are small from the outset; in the second, our regret analysis shows that the excess query cost can vanish relative to $T$.  Conversely, if candidate policies almost never agree on their complete context--action pairs, \TCAB\ remains exact but offers little query saving.  This makes the theory directly diagnostic: the same edge disagreements that determine cost also indicate when feedback sharing will or will not help.

\paragraph{Contributions.}
Our main contributions are as follows.
\begin{enumerate}[leftmargin=2em,itemsep=0.3em]
    \item We formulate exact finite-horizon comparison of $J$ arbitrary history-dependent, possibly randomized contextual policies under i.i.d. full contexts and nonparametric full-context reward kernels $Q_a(\cdot\mid x)$.
    \item We introduce a round-synchronous, predictably time-adaptive version of \TCAB.  Maximal couplings on all selected tree edges coexist, all same-depth coupling operations and all component-level reward queries admit parallel execution, and every policy retains exactly its standalone trajectory law.
    \item We prove the pathwise and expected cost identities
    \[
    N(T)=T+\sum_{t=1}^T\sum_{e\in E_t}D_{e,t},
    \qquad
    \E[N(T)]=T+\sum_{t=1}^T\E\!\left[\sum_{e\in E_t}\delta_{e,t}\right].
    \]
    Within the natural class of conditionally exact edge-local designs, \TCAB\ is optimal on every selected tree.  A baseline star and a current-round minimum-spanning tree give two concrete specializations.
    \item We establish a finite-sample regret-to-cost bound for predictable tree sequences.  For fixed $J$, almost-sure uniqueness of the oracle action and $o(T)$ pseudo-regret for every policy imply $T+o(T)$ expected reward queries; a margin condition yields an explicit rate.
    \item We prove finite-sample variance bounds for every pairwise policy contrast.  The bounds extend to time-varying trees and isolate the roles of edge mismatches and realized pseudo-regret; a general fixed zero-sum extension is given in the appendix.
    \item Across two language-model evaluation tasks and an adaptive search-bandit task, \TCAB\ improves the empirical cost--precision trade-off relative to independent A/B/n baselines with matched or full budgets.
\end{enumerate}

\section{Related literature}

\paragraph{Large-scale online experimentation.}
Online controlled experimentation has become core infrastructure at large digital platforms.  Google's overlapping-experiment architecture was designed to run more experiments on limited traffic by allowing compatible experiments in different layers to overlap on the same users or queries \citep{tang2010overlapping}; Microsoft/Bing and LinkedIn describe related large-scale experimentation systems and organizational challenges \citep{kohavi2013large,xu2015infrastructure}.  A cross-industry summit emphasizes the scale of the resulting experimentation programs and the operational pressure for faster, more reliable decisions \citep{gupta2019challenges}.  These systems motivate our resource question but solve a different problem: they multiplex traffic across experiments, whereas \TCAB\ shares realized feedback across the candidate policies within one multi-policy comparison while preserving each policy's standalone trajectory law.

\paragraph{Artificial Replay and comparison of learning algorithms.}
The closest methodological work is \citet{meng2026design}, which introduces Artificial Replay for comparing two context-free stochastic bandit algorithms and proves exact marginal, interaction-cost, and variance guarantees.  Our setting introduces two obstacles absent from the two-policy context-free construction.  First, a contextual policy selects actions after observing a query, so exact reuse requires coupling the complete context--action law rather than only an arm label.  Second, for $J\ge3$ distributions, all pairwise maximal couplings need not be compatible; our policy tree selects $J-1$ pairs whose maximal couplings can be glued simultaneously.  The phrase ``artificial replay'' is also used by \citet{banerjee2022artificial} for incorporating an exogenous historical data set to warm-start a bandit.  That problem concerns how one learner uses historical observations, rather than how several prospective policy trajectories share newly generated feedback.

\paragraph{Multiple comparisons, adaptive experiments, and policy selection.}
Classical many-to-one procedures compare several static treatments with a common control while accounting for multiplicity \citep{dunnett1955multiple}.  Adaptive-design work studies how observations should be allocated, or how valid inference can be maintained, under adaptively collected data \citep{kasy2021adaptive,hadad2021confidence,simchilevi2025experimental}.  Related work on policy comparison and selection includes safe exploration for evaluating several policies \citep{wan2022safe}, high-confidence off-policy selection \citep{kuzborskij2021confident}, and ranking policies from a fixed experience data set \citep{yang2022offline}.  These works optimize allocation or inference under a given data-generating scheme.  Our objective is different: construct a joint prospective experiment in which every adaptive candidate has exactly the finite-horizon path law it would have had in isolation, but duplicated reward queries are shared whenever coupling permits.

\paragraph{Contextual replay and off-policy evaluation.}
\citet{li2011unbiased} give an exact replay evaluator for a contextual-bandit algorithm from a uniformly randomized log: a logged event is retained when the target algorithm chooses the logged action, so the retained adaptive history has the same law as an online target trajectory.  Inverse-propensity, doubly robust, self-normalized, and related methods form a broader off-policy-evaluation literature \citep{dudik2011doubly,swaminathan2015counterfactual,wang2017optimal,su2020shrinkage,zhan2021adaptive}.  Those methods begin with an exogenous log and therefore require support, weighting, modeling, or a bias--variance compromise.  In \TCAB, the candidate policies are runnable and the experiment prospectively generates residual observations when policies fail to couple.  Lack of overlap therefore raises the number of reward queries rather than invalidating the target policy trajectory.

\paragraph{Data sharing and interference between learning algorithms.}
A distinct line of work asks what happens when algorithms in an A/B experiment share training data.  \citet{brennan2025symbiosis} document ``symbiosis bias'' in recommendation experiments, and \citet{li2025sharing} analyze settings in which data sharing can alter or even reverse the ranking of two bandit algorithms.  Such sharing changes what the learning algorithms observe and can therefore change the estimand.  \TCAB\ instead introduces dependence through an exact coupling: an observation is shared only on a branch for which the recipient policy's conditional transition law remains correct.  Hence cross-policy dependence is intentional and accounted for, while each individual policy's path law is preserved.

\paragraph{Maximal coupling, common random numbers, and multi-marginal transport.}
For two probability laws, maximal coupling attains the total-variation lower bound on disagreement \citep{thorisson2000coupling,lindvall2002lectures}; one-sided rejection constructions provide exact implementations under density-ratio access \citep{corenflos2022coupled}.  With three or more marginals, simultaneous pairwise maximality can fail \citep{angel2019pairwise}.  Our tree construction uses the complementary fact that prescribed pairwise couplings on the edges of an acyclic graph can always be glued, so exactly $J-1$ strategically selected pairs can be maximal without approximation.  Sharing contexts and outcomes is also related to common-random-number methods for simulation comparison \citep{kleijnen1975antithetic,nelson1995common,glasserman2004monte}.  The adaptive-policy setting is more delicate because shared randomness enters a learner's history and changes later actions unless the correct conditional transition is maintained recursively.  At a fixed round, globally minimizing the number of distinct realized context--action pairs over all $J$ marginals can be viewed as a multi-marginal optimal-transport problem \citep{pass2015multimarginal,villani2009optimal}.  We use the tree restriction for exactness, transparency, and an implementable local sampler; we do not claim to solve the unrestricted multi-marginal problem.

\paragraph{Bandits and no-regret online learning.}
Our efficiency guarantees connect directly to the extensive literature on no-regret sequential decision-making. In stochastic multi-armed bandits, classical upper-confidence-bound algorithms achieve logarithmic instance-dependent regret \citep{auer2002finite}, while Thompson sampling admits logarithmic instance-dependent and near-optimal worst-case regret guarantees \citep{agrawal2012analysis,agrawal2013further}. No-regret guarantees extend broadly to contextual and structured settings, including contextual bandits with general policy classes \citep{beygelzimer2011contextual,agarwal2014taming}, linear contextual bandits \citep{abbasi2011improved,agrawal2013thompson}, and stochastic contextual models under structural conditions such as smooth covariate effects, covariate diversity, and plentiful contexts \citep{perchet2013multi,bastani2021mostly,wu2020stochastic,ghosh2022breaking,hanna2023contexts,wen2026optimal}; see \citet{lattimore2020bandit} for a general treatment. Such online-learning models arise in a wide range of applications, including pricing \citep{kleinberg2003value,cohen2020feature}, advertising \citep{weed2016online,wen2025joint,hu2026learning}, and search and recommendation systems \citep{li2010contextual,wen2026marginal}. \TCAB\ does not require the candidate policies to belong to any particular model class and treats them as black boxes: these structural assumptions are needed only to establish a particular policy's regret rate, not for coupling exactness or the query-cost identity. The connection arises through our query-efficiency result: whenever the compared policies have sublinear pseudo-regret and increasingly concentrate on a common optimal action, their pairwise disagreement is sublinear, and \TCAB\ requires only $T+o(T)$ reward queries for any fixed number of policies, rather than the $JT$ queries required by independent evaluation. Thus, advances in no-regret learning that make candidate policies converge more rapidly toward the oracle translate directly into greater feedback sharing under \TCAB.

\section{Problem formulation}
\label{sec:setup}

\subsection{Stochastic contextual environment}

Let $T\ge1$ be the target horizon, $\A=[K]$ the action space, and $(\X,\mathcal B_{\X})$ a standard Borel full-context space.  Contexts arrive as
\begin{equation}
X_t\sim P_X\quad\text{i.i.d. over }t.
\label{eq:context-law}
\end{equation}
For a slate, $X_t$ may contain the user and all candidate-item features; changing feasible actions can likewise be encoded in $X_t$.  For each $a\in\A$, a reward query at $(x,a)$ returns a draw from the full-context kernel $Q_a(\cdot\mid x)$, with mean $\mu_a(x):=\int r\,Q_a(\dd r\mid x)$.

\begin{assumption}[I.i.d. stationary full-context outcomes]
\label{ass:environment}
Conditional on the queried context--action pairs, reward queries are independent draws from their corresponding kernels $Q_a(\cdot\mid x)$; rejected unlabeled context proposals generate no reward and cause no carryover or interference.
\end{assumption}
The full-context specification allows the selected action's outcome to depend on the entire user/slate state and guarantees that two policies matching on $(X,A)$ require the same conditional reward law.

\subsection{Candidate policies and estimands}

Fix $J\ge2$ policies indexed by $\Ical:=\{0,1,\ldots,J-1\}$.  Policy $j$ is a sequence of measurable kernels $\pi_{j,t}(\cdot\mid h,x)\in\Delta(\A)$ with history $H_{t-1}^j=(X_s^j,A_s^j,R_s^j)_{s<t}$; any persistent internal state affecting future decisions is included in this state.  In a standalone run,
\begin{equation}
X_t^j\sim P_X,\qquad
A_t^j\sim\pi_{j,t}(\cdot\mid H_{t-1}^j,X_t^j),\qquad
R_t^j\sim Q_{A_t^j}(\cdot\mid X_t^j).
\label{eq:standalone-transition}
\end{equation}
Define
\begin{equation}
S_j(T):=\sum_{t=1}^T R_t^j,\qquad
V_j(T):=\E[S_j(T)],\qquad
V(T):=(V_0(T),\ldots,V_{J-1}(T))^\top.
\label{eq:values}
\end{equation}
For a baseline $r$, let
\begin{equation}
\theta^{(r)}(T):=\left(V_j(T)-V_r(T)\right)_{j\ne r}.
\label{eq:baseline-contrasts}
\end{equation}
Independent runs spend $JT$ reward queries.  We seek exact horizon-$T$ trajectories for all policies using a random total $N(T)\le JT$, while preserving unbiased estimation of these values and baseline contrasts.

\subsection{Operational access}

\begin{assumption}[Pairwise context--action coupling access]
\label{ass:access}
Before requesting a reward, the experiment can draw an unlabeled $x\sim P_X$, side-effect-free sample/evaluate the relevant policy kernels at $x$, reject a proposal without updating either policy, and request a reward only after a context--action pair is accepted (or reuse an already queried reward after an exact match).  Equivalently, the experiment can implement the pairwise maximal-coupling primitive used below.
\end{assumption}

More explicitly, for an ordered parent--child pair $(p,v)$ at current histories $(h_p,h_v)$, the experiment must be able to (i) draw $x\sim P_X$ before requesting a reward; (ii) sample from and evaluate $\pi_{p,t}(\cdot\mid h_p,x)$ and $\pi_{v,t}(\cdot\mid h_v,x)$ without advancing either policy state; (iii) reject a proposal and return the underlying query to the default workflow without updating either policy; and (iv) after acceptance, either query $Q_a(\cdot\mid x)$ or copy an already queried reward when the two complete pairs match.  The ability to screen contexts without consuming the reward-bearing resource is the operational distinction between a context proposal and an experimental reward query.

We count the latter resource.  Depending on the application, it may correspond to exposing a user to an experimental action and observing a click or purchase, collecting a human preference label, invoking an expensive downstream model or tool, or obtaining another outcome that is materially more costly than evaluating the candidate policies on an unlabeled context.  The number of screened proposals can also matter computationally; Section \ref{sec:implementation} gives its exact rejection-sampling behavior separately from the reward-query count $N(T)$.

Neither the density of $P_X$ nor the reward kernel $Q_a$ needs to be evaluated.  What is needed is policy-probability access on a proposed context.  For deterministic policies, these probabilities reduce to evaluating the policies' chosen actions, and the coupling implementation simplifies accordingly.

\section{Tree-Coupled A/B Testing}
\label{sec:method}

Let $\Fcal_{t-1}$ contain all policy histories and experimental randomness revealed through round $t-1$.  At the start of round $t$, choose an $\Fcal_{t-1}$-measurable rooted spanning tree $\Tcal_t=(\Ical,E_t,r_t)$.  Thus the tree may react to past data but cannot depend on unrevealed round-$t$ information.  Write $p_t(v)$ and $\operatorname{dep}_t(v)$ for the parent and depth of a nonroot vertex $v$.

\subsection{Conditional context--action laws and tree coupling}

Let $\Zcal:=\X\times\A$, $Z_t^j=(X_t^j,A_t^j)$, and, conditional on policy history $h_j$, define the policy's next complete-pair law
\begin{equation}
\nu_{j,t}^{h_j}(\dd x,a):=P_X(\dd x)\,\pi_{j,t}(a\mid h_j,x).
\label{eq:nu-j}
\end{equation}
For policies $i,j$ let
\begin{equation}
\delta_{ij,t}(h_i,h_j)
:=\TV\!\left(\nu_{i,t}^{h_i},\nu_{j,t}^{h_j}\right)
=\frac12\int_{\X}\!\norm{\pi_{i,t}(\cdot\mid h_i,x)-\pi_{j,t}(\cdot\mid h_j,x)}_1P_X(\dd x).
\label{eq:pair-tv}
\end{equation}
A maximal coupling of these two laws matches their complete pairs with probability $1-\delta_{ij,t}$.  Conditional on $\Fcal_{t-1}$ and $\Tcal_t$, \TCAB\ samples the root from its law \eqref{eq:nu-j} and recursively maximally couples each child to its parent.  Since the selected edges form a tree, the $J-1$ prescribed edge couplings can be glued into a single joint law while preserving every node marginal; the formal gluing statement is proved in Appendix \ref{app:coupling}.

The use of the complete pair $Z=(X,A)$ is essential.  Sharing merely because two policies select the same arm can be invalid in a contextual problem: the distribution of the context attached to that arm is itself policy dependent.  With the full-context reward model $Q_a(\cdot\mid x)$, equality of $(X,A)$ is exactly the event on which two policies require the same conditional reward law.

\subsection{One-sided rejection implementation}
\label{sec:implementation}

Every law in \eqref{eq:nu-j} is dominated by $\rho:=P_X\otimes\mathrm{counting}$, with density
\[
p_{j,t}^{h_j}(x,a)=\pi_{j,t}(a\mid h_j,x).
\]
Consider an oriented tree edge $(p,v)$ and fix the current histories $(h_p,h_v)$.  A parent-first maximal coupling first draws $Z_p\sim\nu_{p,t}^{h_p}$ and lets the child inherit this same pair with probability
\begin{equation}
q_{pv,t}(Z_p)
:=\min\!\left\{1,
\frac{p_{v,t}^{h_v}(Z_p)}{p_{p,t}^{h_p}(Z_p)}\right\}
=\min\!\left\{1,
\frac{\pi_{v,t}(A_p\mid h_v,X_p)}{\pi_{p,t}(A_p\mid h_p,X_p)}\right\}.
\label{eq:edge-first-accept}
\end{equation}
The denominator is positive almost surely on a pair generated by the parent.  The accepted mass is the overlap measure $\min\{p_{p,t}^{h_p},p_{v,t}^{h_v}\}\rho$.

If the parent pair is rejected, the child must be drawn from the residual part of its own law.  Operationally, repeatedly draw
\[
X\sim P_X,\qquad A\sim\pi_{v,t}(\cdot\mid h_v,X),
\]
and accept the proposal with probability
\begin{equation}
s_{pv,t}(X,A)
:=1-\min\!\left\{1,
\frac{\pi_{p,t}(A\mid h_p,X)}{\pi_{v,t}(A\mid h_v,X)}\right\}.
\label{eq:edge-residual-accept}
\end{equation}
The accepted residual has normalized density
\[
\frac{(p_{v,t}^{h_v}-p_{p,t}^{h_p})_+}{\delta_{pv,t}(h_p,h_v)}
\]
with respect to $\rho$.  Thus the procedure is an exact maximal coupling.  The common context measure cancels from both ratios, which is why the experiment never needs to know a density for $P_X$; the reward kernel also does not enter the coupling decision.

For deterministic policies, \eqref{eq:edge-first-accept} reduces to an intuitive rule: the parent's pair is inherited exactly when the child would choose the same action on the parent's context.  After a mismatch, contexts are proposed until the two current decision rules disagree in the direction required by the residual law.  This is also useful operationally because model scoring or policy evaluation can often be done before committing a query to the experimental treatment.

\begin{proposition}[Exact rejection implementation and proposal overhead]
\label{prop:implementation}
Let $U_{pv,t}$ denote the number of sampled proposals before one is accepted by \eqref{eq:edge-residual-accept}. Under Assumption \ref{ass:access}, \eqref{eq:edge-first-accept}--\eqref{eq:edge-residual-accept} implement the edge couplings used by \TCAB.  For an edge $(p,v)$ with $\delta_{pv,t}>0$, conditional on entering the residual branch, the number of child proposals is geometric with mean $1/\delta_{pv,t}$.  Since the residual branch itself is entered with probability $\delta_{pv,t}$, its unconditional expected proposal contribution is at most one:
\[
\E[U_{pv,t}] = \begin{cases}
    1, &\text{if $\delta_{pv,t}>0$;}\\
    0, &\text{if $\delta_{pv,t}=0$.}
\end{cases}
\]
Hence the total expected residual-proposal overhead is at most $(J-1)T$.
\end{proposition}

The proposition separates two notions of cost.  Small total variation means that reward sharing is frequent, but a rare residual event can require many screened proposals conditional on occurring.  Their product is controlled: the expected number of residual context proposals per active edge-round is at most one.  Our theoretical objective $N(T)$ nevertheless counts only reward-bearing experimental interactions, because those are the resource that feedback sharing is designed to reduce.

\subsection{Reward inheritance, matched components, and parallelism}

After all context--action pairs at round $t$ are generated, retain the tree edges on which the endpoints match and let $\mathcal C_t$ be the resulting connected components.  Equality propagates along a matched path, so every policy in a component $C$ has the same pair $Z_t^C=(X_t^C,A_t^C)$.  \TCAB\ draws one reward
\[
R_t^C\sim Q_{A_t^C}(\cdot\mid X_t^C)
\]
and assigns that same physical draw to every policy in $C$.  Distinct components receive conditionally independent fresh rewards.  Equivalently, a child inherits its parent's reward on a matched edge and opens a new reward lineage on a mismatch.

The round-synchronous order is important for adaptive policies: all policies first generate their round-$t$ observations and then update to round $t+1$.  Within a round, however, every child at the same tree depth can be coupled in parallel once its parent pair is available, and the reward queries for distinct matched components can also be issued in parallel.  For nonadaptive policies, different rounds may additionally be parallelized because there is no history dependence across $t$.

\begin{algorithm}[t]
\caption{Round-synchronous Tree-Coupled A/B Testing (\TCAB)}
\label{alg:tcab}
\KwIn{Horizon $T$, policies $(\pi_j)_{j\in\Ical}$, predictable tree rule $(\Psi_t)_{t\le T}$, context sampler, reward interface}
Initialize $H_0^j\gets\varnothing$ for every $j\in\Ical$\;
\For{$t=1,\ldots,T$}{
    Select $\Tcal_t=(\Ical,E_t,r_t)\gets\Psi_t((H_{t-1}^j)_{j\in\Ical})$\;
    Draw $Z_t^{r_t}\sim\nu_{r_t,t}^{H_{t-1}^{r_t}}$\;
    \For{$d=1,\ldots,\max_v\operatorname{dep}_t(v)$}{
        In parallel over $v$ at depth $d$, maximally couple $Z_t^v$ to $Z_t^{p_t(v)}$ using \eqref{eq:edge-first-accept}--\eqref{eq:edge-residual-accept}\;
    }
    Form the matched-edge components $\mathcal C_t$\;
    In parallel over $C\in\mathcal C_t$, query one $R_t^C\sim Q_{A_t^C}(\cdot\mid X_t^C)$ and assign it to every $j\in C$\;
    Simultaneously append $(X_t^j,A_t^j,R_t^j)$ to every policy history\;
}
\KwRet{$(S_j(T))_{j\in\Ical}$ and the desired policy contrasts.}
\end{algorithm}

\subsection{Tree choices and when they help}
\label{sec:tree-choices}

A \emph{star} rooted at a designated baseline $r$ is the simplest choice.  It makes every baseline--alternative pair maximal and has depth one, so all alternatives can be processed in parallel after the baseline pair is available.  Its cost, however, charges disagreement with the baseline $J-1$ times; it is therefore most attractive when the incumbent policy is representative of the alternatives or when baseline-versus-alternative contrasts are the main inferential targets.  We denote this specialization by \TCABstar.

A general tree can connect similar policies through intermediate policies.  If all current pairwise distances $\delta_{ij,t}(H_{t-1}^i,H_{t-1}^j)$ were available, a minimum-spanning tree minimizes their sum and hence the current-round expected query cost within the edge-local tree class; Section \ref{sec:results} states this formally.  Deterministic tie-breaking makes such a rule predictable.  The guarantee is deliberately myopic: today's joint coupling can affect cross-policy dependence in future histories, so a sequence of greedy MSTs need not be globally horizon-optimal.

In practice, current conditional TV distances are rarely available exactly.  Our \TCABmst\ implementation uses pilot data or a simulator to estimate a fixed pairwise similarity score---for example, average complete-pair mismatch over pilot trajectories---constructs the corresponding MST, and freezes it before the main experiment.  For adaptive policies the pilot score can depend on the pilot history distribution and coupling, so pilot construction, random seeds, and separation from the main experiment should be reported.

These two choices illustrate where the method gains most.  A star is natural for product launches that compare many challengers to a deployed incumbent.  An MST is natural when candidates form clusters---for example, nearby checkpoints, prompt variants, exploration parameters, or recommendation rules---because an intermediate policy can relay feedback between alternatives that are not both close to the baseline.  The exact cost identity below makes this intuition quantitative: only disagreement on the selected tree edges creates additional reward queries.

\section{Exactness and reward-query efficiency}
\label{sec:results}

For $e=\{i,j\}\in E_t$, define
\[
D_{e,t}:=\ind\{Z_t^i\ne Z_t^j\},
\qquad
N_t:=\text{number of reward queries at round }t,
\qquad
N(T):=\sum_{t=1}^T N_t.
\]
The component construction gives
\begin{equation}
N_t=1+\sum_{e\in E_t}D_{e,t},
\qquad
N(T)=T+\sum_{t=1}^T\sum_{e\in E_t}D_{e,t}
\quad\text{pathwise}.
\label{eq:pathwise-tree-cost}
\end{equation}
Since every tree has $J-1$ edges,
\begin{equation}
T\le N(T)\le JT
\quad\text{almost surely}.
\label{eq:query-range}
\end{equation}

\subsection{Exact trajectory marginals and unbiased contrasts}

\begin{theorem}[Exact marginal trajectories for predictable tree sequences]
\label{thm:marginals}
Under Assumption \ref{ass:environment}, for every finite $T,J$, every predictable rooted-tree sequence $(\Tcal_t)_{t\le T}$, and every collection of measurable history-dependent policies, the \TCAB\ trajectory
\[
H_T^j=(X_t^j,A_t^j,R_t^j)_{t=1}^T
\]
has exactly the same distribution as the standalone process \eqref{eq:standalone-transition}, for every $j\in\Ical$.
\end{theorem}

The theorem is a complete path-law statement.  Policies are dependent across $j$, but no individual policy can statistically distinguish its \TCAB\ trajectory from a standalone run.

\begin{corollary}[Unbiased values and baseline contrasts]
\label{cor:unbiased}
If rewards have finite first moments, then
\[
\E[S_j(T)]=V_j(T),
\qquad
\E[S_j(T)-S_r(T)]=V_j(T)-V_r(T)
\]
for every $j,r\in\Ical$.
\end{corollary}

\subsection{Exact query cost and the scope of optimality}

The lower-bound statement requires a conditional notion of exactness.  Marginal equality of a policy's complete trajectory alone does not imply that its next-step law remains correct after conditioning on other policies' histories.

\begin{definition}[Conditionally exact design]
\label{def:conditional-exact}
A multi-policy design is \emph{conditionally exact} if, for every round $t$ and policy $j$,
\[
\mathcal L(Z_t^j\mid\Fcal_{t-1})
=\nu_{j,t}^{H_{t-1}^j}
\quad\text{almost surely},
\]
and the reward assigned to $j$, conditional on $\Fcal_{t-1}$ and $Z_t^j=(x,a)$, has law $Q_a(\cdot\mid x)$.  The design is \emph{edge-local} on a rooted tree if a nonroot policy can inherit an existing reward lineage only from its parent and only when the two complete pairs coincide; otherwise the child opens a fresh reward query that may be inherited by its descendants.
\end{definition}

\begin{theorem}[Conditional cost identity and edge-local optimality]
\label{thm:tree-cost}
Let
\[
\delta_{e,t}
:=\TV\!\left(
\nu_{i,t}^{H_{t-1}^i},
\nu_{j,t}^{H_{t-1}^j}
\right),
\qquad e=\{i,j\}\in E_t.
\]
For \TCAB,
\begin{equation}
\E[N_t\mid\Fcal_{t-1}]
=1+\sum_{e\in E_t}\delta_{e,t},
\qquad
\E[N(T)]
=T+\sum_{t=1}^T\E\!\left[\sum_{e\in E_t}\delta_{e,t}\right].
\label{eq:tree-cost-identity}
\end{equation}
Conditional on the same pre-round information and selected tree, every conditionally exact edge-local design has expected round-$t$ cost at least $1+\sum_{e\in E_t}\delta_{e,t}$.  Thus \TCAB\ is conditionally optimal in that class.
\end{theorem}

This theorem does not claim global instancewise optimality among arbitrary multi-marginal designs.  A non-tree design may sometimes share a reward across nonadjacent policies and beat every tree.  The tree restriction buys an explicit sampler, simultaneous edge maximality, and exact cost accounting.

\begin{corollary}[Baseline-centered star]
\label{cor:star-cost}
For the fixed star rooted at $r$,
\begin{equation}
\E[N(T)]
=T+\sum_{t=1}^T\sum_{j\ne r}
\E\!\left[
\TV\!\left(
\nu_{r,t}^{H_{t-1}^r},
\nu_{j,t}^{H_{t-1}^j}
\right)
\right].
\label{eq:star-cost}
\end{equation}
Within the class of conditionally exact baseline-local designs, \TCABstar\ minimizes each round's conditional expected cost.
\end{corollary}

\begin{corollary}[Myopic MST optimality within tree designs]
\label{cor:mst-tree}
Conditional on $\Fcal_{t-1}$, suppose every pairwise distance $\delta_{ij,t}$ is available.  The smallest round-$t$ conditional expected cost among conditionally exact edge-local tree designs is
\begin{equation}
1+
\min_{\Tcal\text{ spanning tree}}
\sum_{\{i,j\}\in E(\Tcal)}\delta_{ij,t}.
\label{eq:mst-cost}
\end{equation}
A predictably selected minimum-spanning tree followed by tree-maximal coupling attains this value.  The result is current-round optimality; it does not establish full-horizon optimality of the greedy MST sequence.
\end{corollary}

When a fixed star uses a baseline trajectory from historical logs, the leading $T$ queries may be avoided only if the log itself has the exact standalone law required here and the policies can be replayed against the stored baseline history.  This observation does not automatically extend to a changing-root adaptive tree.

\subsection{Low regret implies near-one-trajectory cost}

Assume rewards lie in $[0,1]$ for this subsection.  Define
\[
\mu_*(x):=\max_{a\in\A}\mu_a(x),
\qquad
\Delta_a(x):=\mu_*(x)-\mu_a(x).
\]
When the maximizer is unique, let
\[
a^*(x):=\argmax_a\mu_a(x),
\qquad
\Delta(x):=\min_{a\ne a^*(x)}\Delta_a(x).
\]
Policy $j$'s pseudo-regret is
\begin{equation}
\mathcal R_j(T)
:=\E\!\left[\sum_{t=1}^T\Delta_{A_t^j}(X_t^j)\right].
\label{eq:policy-regret}
\end{equation}
Let $d_{j,t}$ be the degree of $j$ in $\Tcal_t$ and define the degree-weighted regret
\begin{equation}
\mathcal R_{\deg}(T)
:=\E\!\left[
\sum_{t=1}^T\sum_{j\in\Ical}
 d_{j,t}\Delta_{A_t^j}(X_t^j)
\right].
\label{eq:degree-regret}
\end{equation}
Since $d_{j,t}\le J-1$,
\begin{equation}
\mathcal R_{\deg}(T)
\le (J-1)\sum_{j\in\Ical}\mathcal R_j(T).
\label{eq:degree-regret-upper}
\end{equation}

\begin{theorem}[Finite-sample regret-to-cost bound]
\label{thm:regret-cost}
Suppose the oracle action is unique $P_X$-almost surely and define
\[
F_\Delta(\eps):=\Pp\{\Delta(X)\le\eps\}.
\]
For every predictable tree sequence and every $\eps>0$,
\begin{equation}
\E[N(T)]-T
\le
2(J-1)T F_\Delta(\eps)
+\frac{\mathcal R_{\deg}(T)}{\eps}.
\label{eq:finite-regret-cost}
\end{equation}
Consequently, if $J$ is fixed and $\mathcal R_j(T)=o(T)$ for every $j$, then
\begin{equation}
\E[N(T)]=T+o(T).
\label{eq:regret-near-one}
\end{equation}
\end{theorem}

Almost-sure uniqueness is needed only to force $F_\Delta(\eps)\to0$ as $\eps\downarrow0$.  Without a common tie-breaking rule, two zero-regret policies can disagree indefinitely on a positive-probability tie set.

\begin{corollary}[Margin and uniform-gap rates]
\label{cor:margin-rate}
If $F_\Delta(\eps)\le C\eps^\beta$ for all sufficiently small $\eps$, with $C,\beta>0$, then
\begin{equation}
\E[N(T)]-T
=O\!\left(
\left((J-1)T\right)^{1/(\beta+1)}
\mathcal R_{\deg}(T)^{\beta/(\beta+1)}
\right),
\label{eq:margin-rate}
\end{equation}
where the constant depends only on $C$ and $\beta$.  If $\Delta(X)\ge\Delta_{\min}>0$ almost surely, then
\begin{equation}
\E[N(T)]-T
\le\frac{\mathcal R_{\deg}(T)}{\Delta_{\min}}.
\label{eq:uniform-gap-rate}
\end{equation}
\end{corollary}

\subsection{Variance reduction for policy contrasts}
\label{sec:variance}

Assume rewards lie in $[0,1]$.  Define realized regret
\[
L_j(T):=\sum_{t=1}^T\Delta_{A_t^j}(X_t^j).
\]
For policies $i,j$, let $P_{ij,t}$ denote their unique path in the round-$t$ tree $\Tcal_t$ and set
\begin{equation}
B_{ij}(T)
:=15\sum_{t=1}^T
\E\!\left[\sum_{e\in P_{ij,t}}D_{e,t}\right]
+6\Var(L_i(T))+6\Var(L_j(T)).
\label{eq:pair-variance-proxy}
\end{equation}

\begin{theorem}[Finite-sample pairwise-contrast variance]
\label{thm:pairwise-variance}
For every $i,j\in\Ical$,
\begin{equation}
\Var\!\left(S_j(T)-S_i(T)\right)
\le B_{ij}(T).
\label{eq:path-variance-bound}
\end{equation}
In particular, if the expected number of mismatches along the time-varying paths $P_{ij,t}$ is $o(T)$ and both realized-regret variances are $o(T)$, then
\[
T^{-1}\Var\!\left(S_j(T)-S_i(T)\right)\longrightarrow0.
\]
\end{theorem}

Theorem \ref{thm:pairwise-variance} is the main inferential guarantee: a round contributes no reward noise to the $i$--$j$ comparison whenever the entire tree path between the two policies matches.  Any fixed zero-sum linear comparison can be written as a finite linear combination of pairwise contrasts; Appendix \ref{app:zero-sum} records the resulting general extension.

\section{Numerical experiments}

This section evaluates the empirical cost--precision trade-off of \TCAB\ using the two specializations \TCABstar\ and \TCABmst\ in three settings.  For every horizon and Monte Carlo replication, we run one faithful horizon-$T$ tree experiment and record its realized reward-query cost $N(T)$; we do not average multiple tree trajectories within a reported experimental unit.  We compare against independent A/B/n testing with the full $JT$ budget (\texttt{AB.FULL}) and against a matched-budget independent design.  The latter allocates $q_T=\lfloor N(T)/J\rfloor$ observations or learning rounds to each policy.  For history-free policies, scaling the resulting per-round estimate by $T$ remains unbiased for the horizon-$T$ cumulative target; for adaptive policies, the corresponding extrapolation is generally biased and is interpreted separately below.

Across the three experiments, \TCAB\ improves the cost--precision frontier while preserving the target contrasts.  At matched query budgets, it reduces median contrast error and variance relative to independent A/B/n designs, detects smaller policy differences, and uses substantially fewer reward observations.  The results show that STAR and MST coupling can make multi-policy comparisons practical under constrained experimental budgets.

\subsection{LLM agents as nonadaptive policies}

\paragraph{RewardBench}
RewardBench is a benchmark for assessing whether language models correctly rank a human--preferred response above a rejected response \citep{lambert-etal-2025-rewardbench}. We use all 2,985 examples in its filtered evaluation split, which span chat, difficult chat, safety, mathematical reasoning, and code, and sample examples according to weights that reproduce the benchmark's official category--level aggregation. Each example constitutes a context \(X\) containing a prompt and a preferred--rejected response pair, and the action space is binary: \(A=0\) selects the preferred response and \(A=1\) selects the rejected response. The policies are $J=12$ language models with complete scores in the pinned RewardBench results repository. For policy \(j\), we convert its cached scores into the deterministic action \(A_j(X)=0\) when it assigns the higher score to the preferred response and \(A_j(X)=1\) otherwise, using a fixed tie rule. We then model a noisy preference observation by assigning reward \(R\sim\mathrm{Bernoulli}(0.9)\) when the selected response is preferred and \(R\sim\mathrm{Bernoulli}(0.1)\) otherwise.

\paragraph{MMLU-Pro}
MMLU-Pro is a challenging multiple-choice benchmark for evaluating language understanding and reasoning across 14 academic domains, with questions containing up to ten candidate answers \citep{wang2024mmlupro}. We use its pinned test split as a finite empirical context distribution: a context $X$ is a question together with its answer options, sampled uniformly with replacement, and the action space \(A=0,\dots,9\) indexes the candidate answers. We construct $J=6$ deterministic policies from three open-weight instruction models---SmolLM2-360M-Instruct, Qwen2.5-0.5B-Instruct, and TinyLlama-1.1B-Chat---each evaluated under plain and chat-formatted prompt variants. For every model--prompt combination and question, we score the answer letters and cache the highest-scoring valid option as the policy action. The reward is \(R(X,A)=\mathbf{1}\{A=Y_X\}\), where \(Y_X\) is the correct answer, so each policy's value equals its test-set accuracy.

\begin{figure*}[ht]
\centering
\begin{subfigure}[t]{0.24\textwidth}
  \centering
  \includegraphics[width=\linewidth]
    {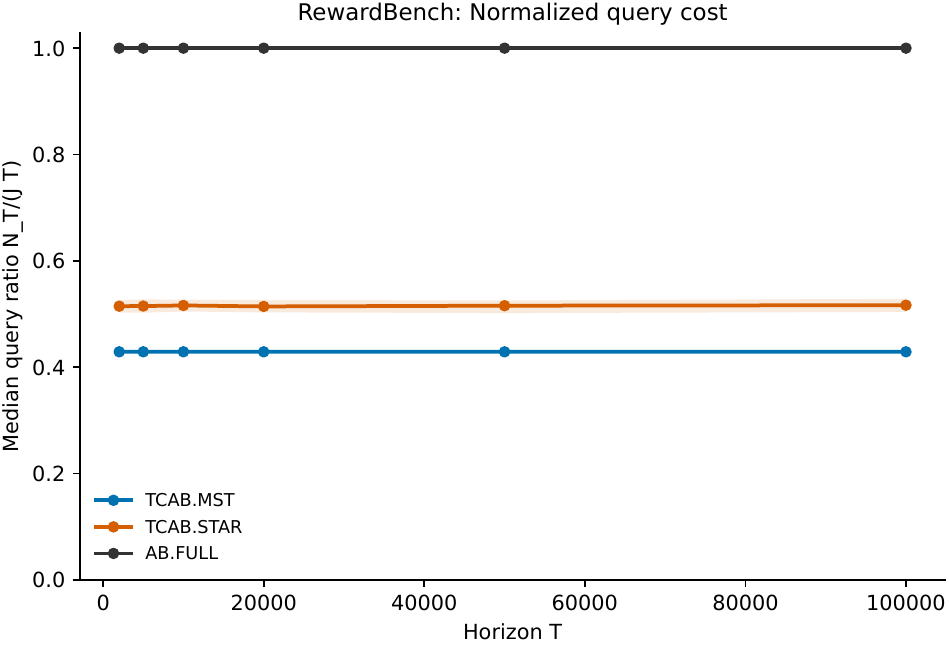}
  \caption{query ratio.}
  \label{fig:rb-query-ratio}
\end{subfigure}\hfill
\begin{subfigure}[t]{0.24\textwidth}
  \centering
  \includegraphics[width=\linewidth]
    {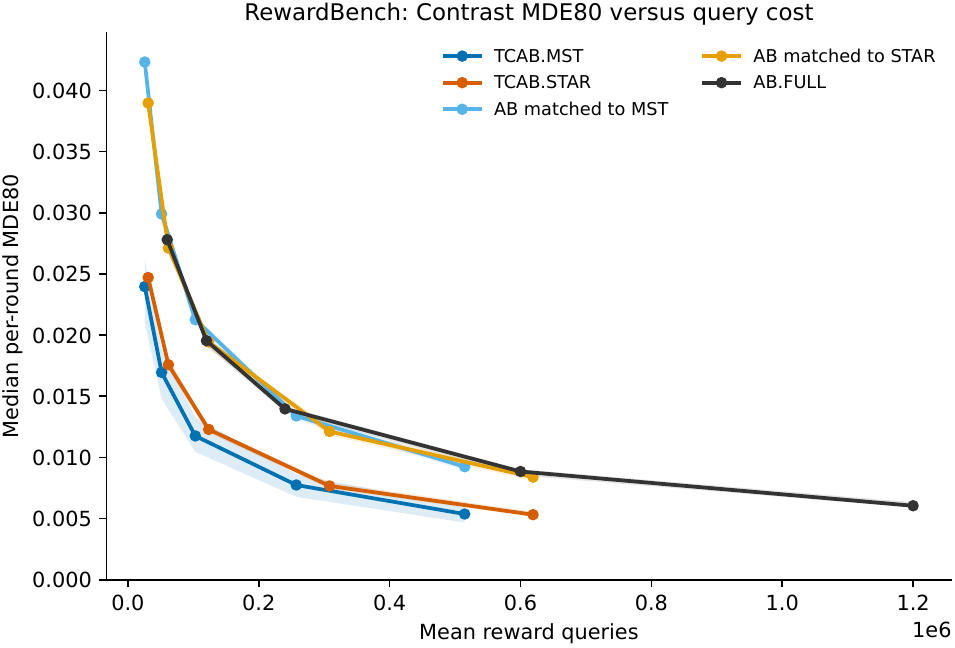}
  \caption{MDE80 vs query cost.}
  \label{fig:rb-mde}
\end{subfigure}\hfill
\begin{subfigure}[t]{0.24\textwidth}
  \centering
  \includegraphics[width=\linewidth]
    {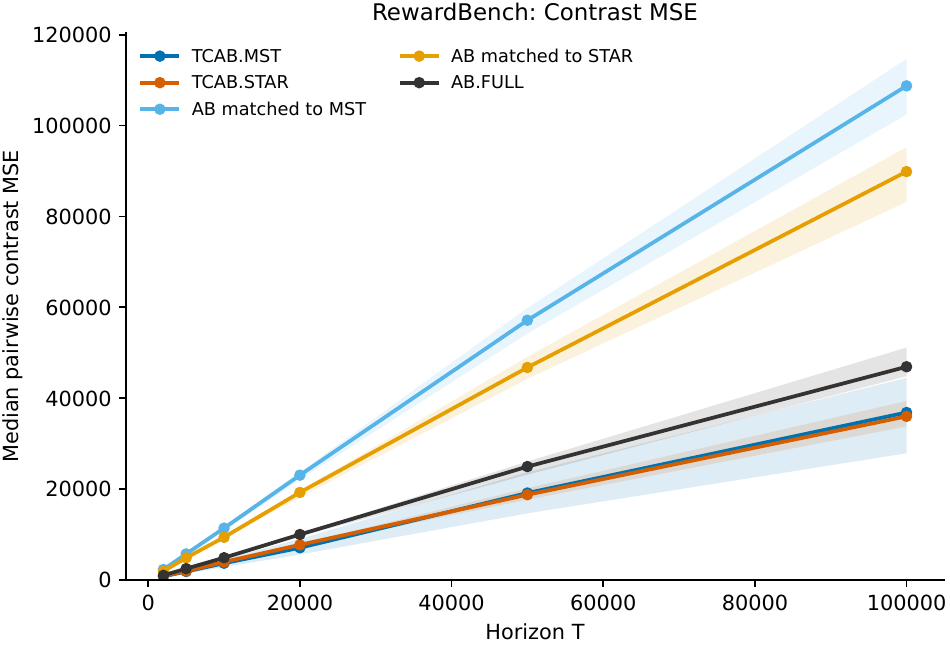}
  \caption{contrast MSE.}
  \label{fig:rb-mse}
\end{subfigure}
\hfill
\begin{subfigure}[t]{0.24\textwidth}
  \centering
  \includegraphics[width=\linewidth]
    {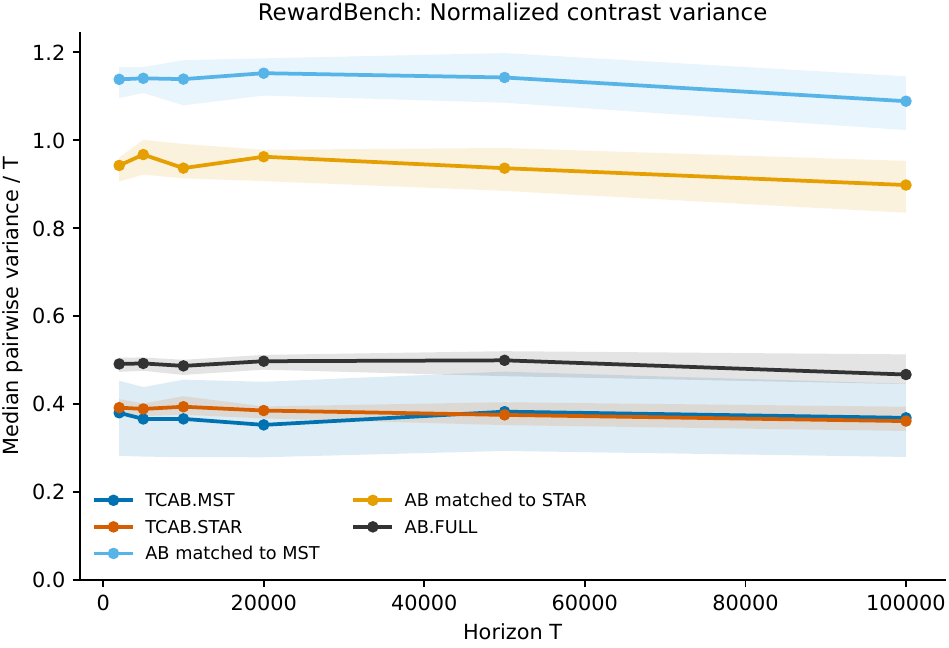}
  \caption{contrast variance.}
  \label{fig:rb-var}
\end{subfigure}

\smallskip

\begin{subfigure}[t]{0.24\textwidth}
  \centering
  \includegraphics[width=\linewidth]
    {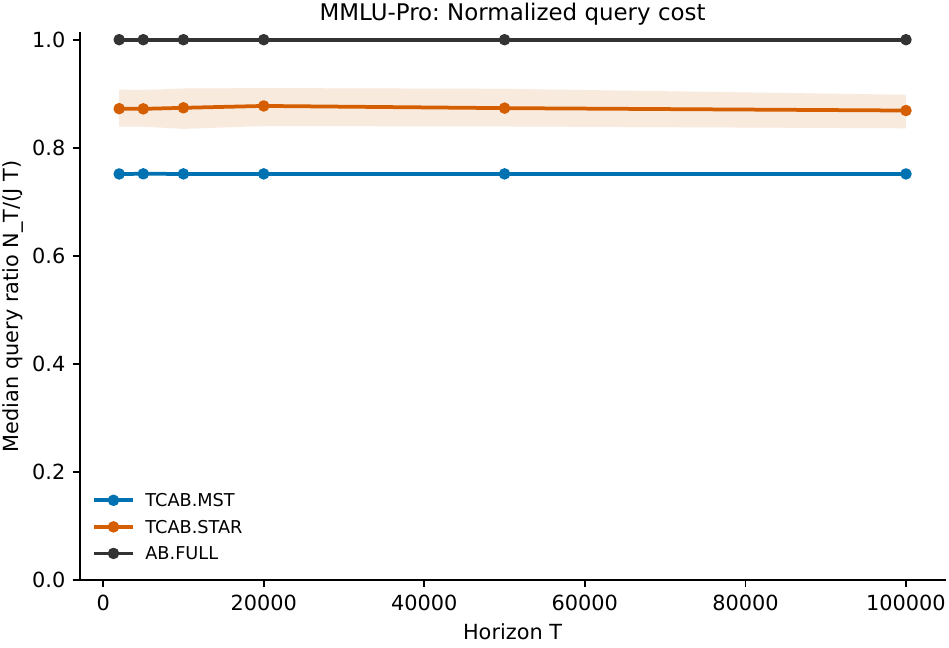}
  \caption{query ratio.}
  \label{fig:mmlu-query-ratio}
\end{subfigure}\hfill
\begin{subfigure}[t]{0.24\textwidth}
  \centering
  \includegraphics[width=\linewidth]
    {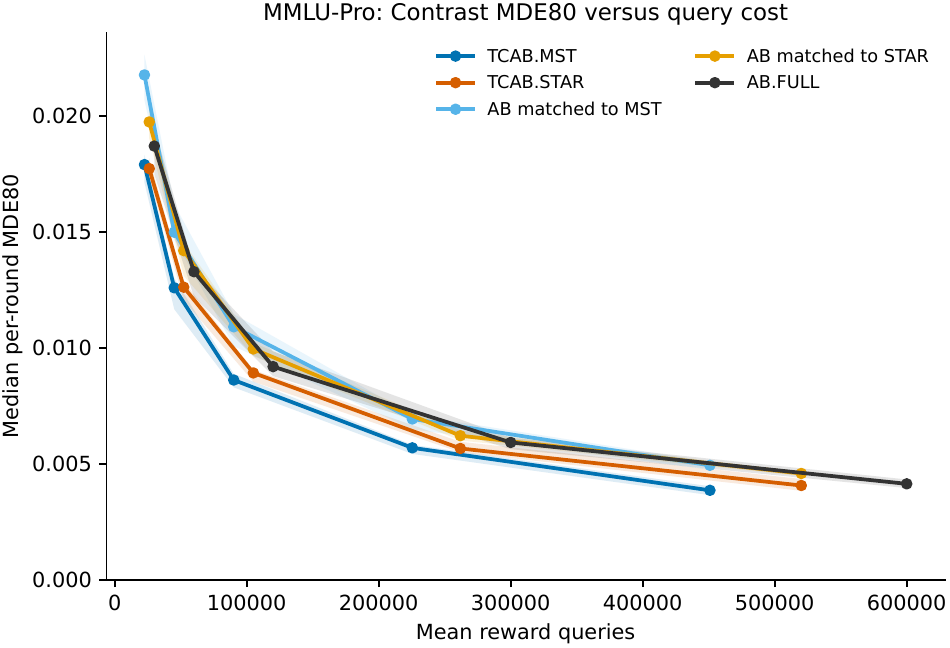}
  \caption{MDE80 vs query cost.}
  \label{fig:mmlu-mde}
\end{subfigure}\hfill
\begin{subfigure}[t]{0.24\textwidth}
  \centering
  \includegraphics[width=\linewidth]
    {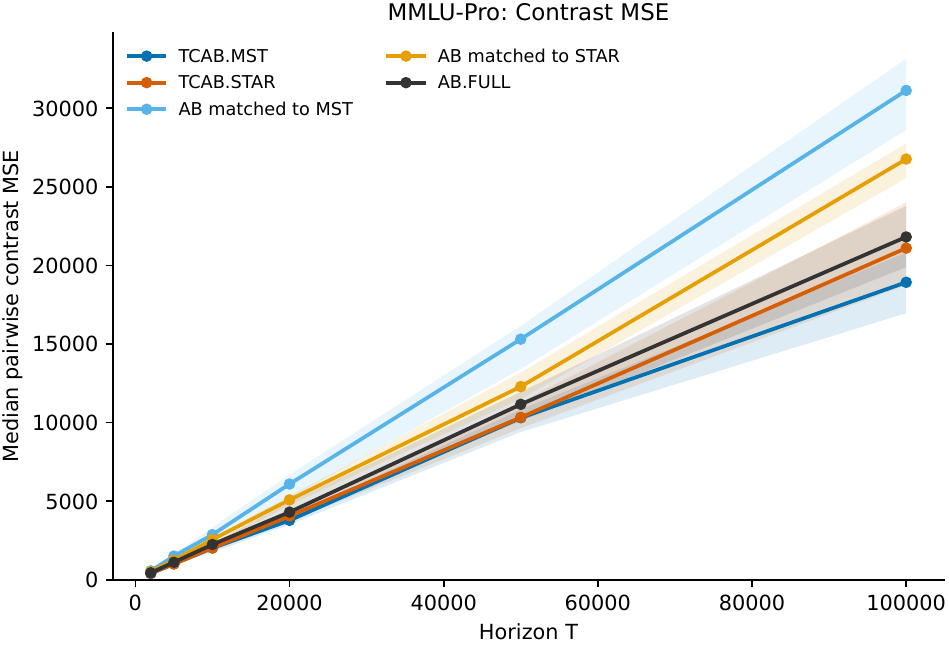}
  \caption{contrast MSE.}
  \label{fig:mmlu-mse}
\end{subfigure}\hfill
\begin{subfigure}[t]{0.24\textwidth}
  \centering
  \includegraphics[width=\linewidth]
    {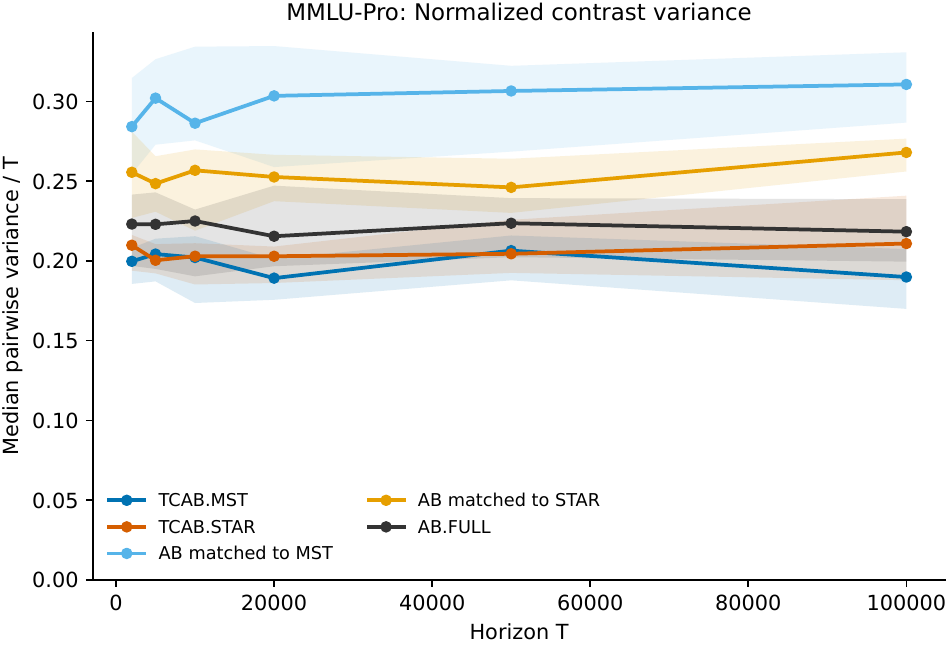}
  \caption{contrast variance.}
  \label{fig:mmlu-var}
\end{subfigure}

\caption{
  Query cost, inferential sensitivity, and estimation accuracy for RewardBench (top) and MMLU-Pro (bottom). Lines report medians and shaded regions report the interquartile range across pairwise policy contrasts over 500 independent runs.
  Lower MDE80 indicates sensitivity to smaller policy differences.
}
\label{fig:fixed-policy-cost-sensitivity}
\end{figure*}

The minimum detectable effect at 80\% power (MDE80) is the smallest per-round difference between two policies that a two-sided 5\% test detects with 80\% probability under the testing procedure used in the experiment. Lower values indicate better sensitivity. Figure \ref{fig:fixed-policy-cost-sensitivity} shows that \TCAB\ resolves smaller differences at a given reward-query budget. Both \TCABstar\ and \TCABmst\ achieve cumulative mean-squared error (MSE) and variance comparable to \texttt{AB.FULL}, and substantially smaller than independent A/B/n runs with matched budgets, while using about 80\% and 40\% of the full query budget on RewardBench and MMLU-Pro, respectively.\footnote{For history-free, time-invariant policies, every edge distance is a constant $\delta_e$, so Theorem \ref{thm:tree-cost} gives $\E[N(T)]=T\{1+\sum_{e\in E}\delta_e\}$. Hence the normalized query ratio is constant in $T$, as observed.} On RewardBench, the reported median contrast MSE and variance reductions relative to matched-budget A/B/n are about 67\% for \TCABmst\ and 59\% for \TCABstar. On MMLU-Pro, the corresponding reductions are about 31\% and 20\%.

\subsection{Adaptive online learning policies}

To further highlight the benefit of \TCAB, we consider a benchmark with adaptive policies. By Theorem \ref{thm:regret-cost}, when the policies asymptotically learn the optimal action, we are able to share most of the observations and attain $\E[N(T)]=T+o(T)$.

\paragraph{MSLR-Search}
We construct a semi-synthetic search environment from MSLR-WEB10K, a learning-to-rank dataset containing query--document relevance judgments and 136 ranking features \citep{microsoft2010mslr,QinL13}. At each round, a query is sampled uniformly and ten candidate documents are drawn from its precomputed top-20 pool; the resulting document-feature vectors form the context, and the action selects one document for the top search position. Selecting document $a$ produces a Bernoulli click whose mean is a ridge-regression relevance score, constrained to \([0.05,0.95]\). We compare six adaptive policies---LinUCB with exploration parameters \(0.5\) and \(1.0\), decaying \(\epsilon\)-greedy with constants \(0.5\) and \(2.0\), and finite-particle linear Thompson sampling with scales \(0.1\) and \(0.25\). Because these policies update from their observed clicks, this experiment evaluates \TCAB\ in a genuinely history-dependent setting; the semi-synthetic ridge-regression score permits accurate hindsight estimates of policy values and cumulative regret.

\begin{figure*}[ht]
\centering
\begin{subfigure}[t]{0.24\textwidth}
  \centering
  \includegraphics[width=\linewidth]
    {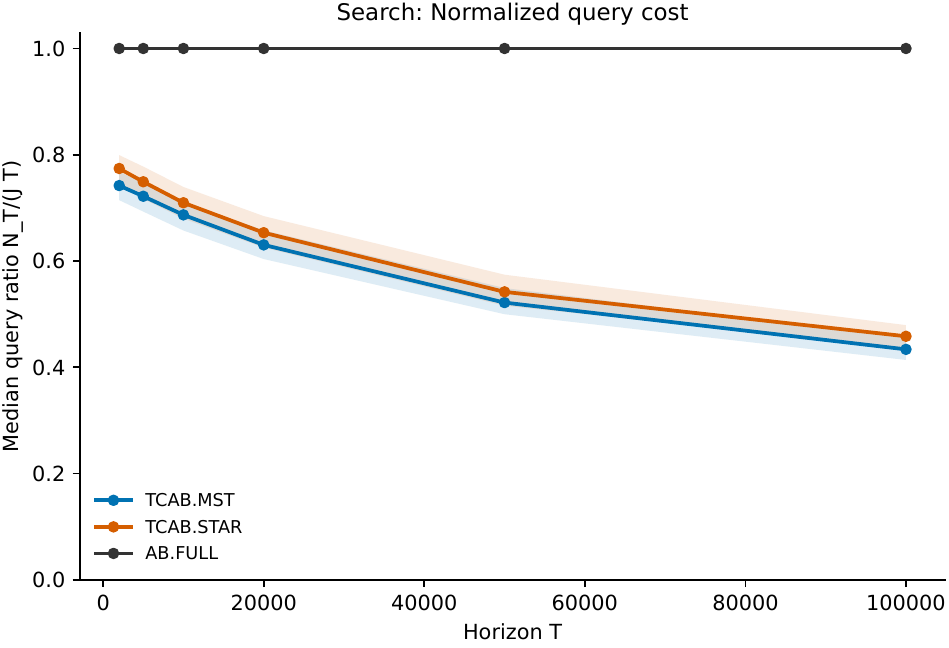}
  \caption{query ratio.}
  \label{fig:search-query-ratio}
\end{subfigure}\hfill
\begin{subfigure}[t]{0.24\textwidth}
  \centering
  \includegraphics[width=\linewidth]
    {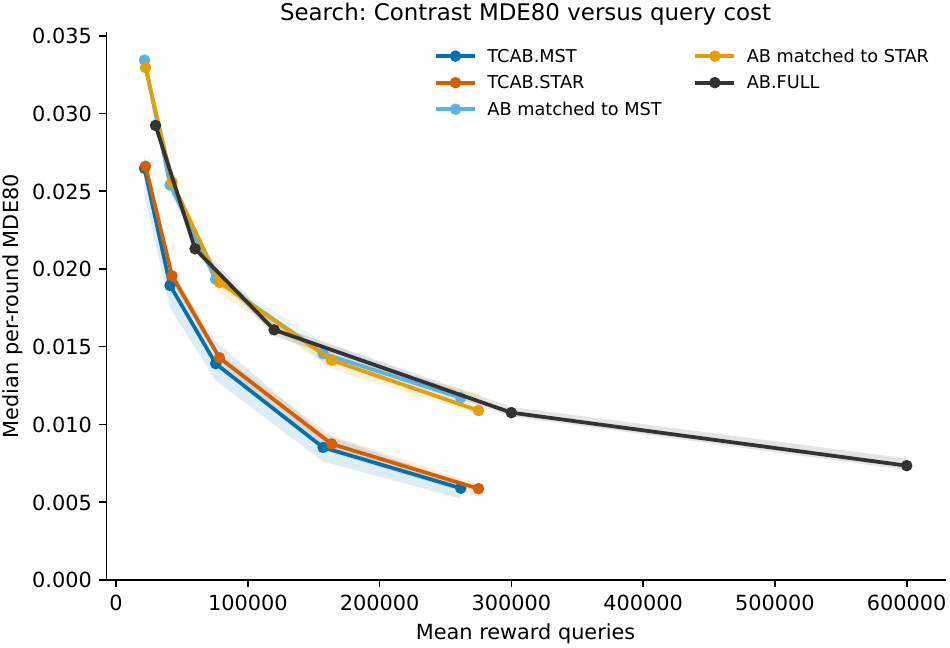}
  \caption{MDE80 vs query cost.}
  \label{fig:search-mde}
\end{subfigure}
\hfill
\begin{subfigure}[t]{0.24\textwidth}
  \centering
  \includegraphics[width=\linewidth]
    {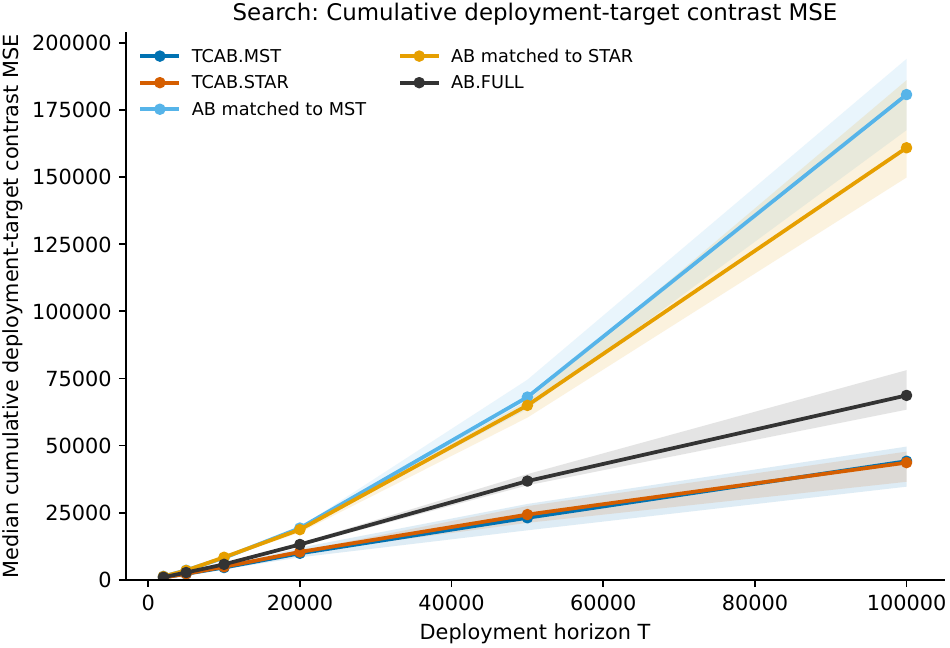}
  \caption{contrast MSE.}
  \label{fig:search-mse}
\end{subfigure}\hfill
\begin{subfigure}[t]{0.24\textwidth}
  \centering
  \includegraphics[width=\linewidth]
    {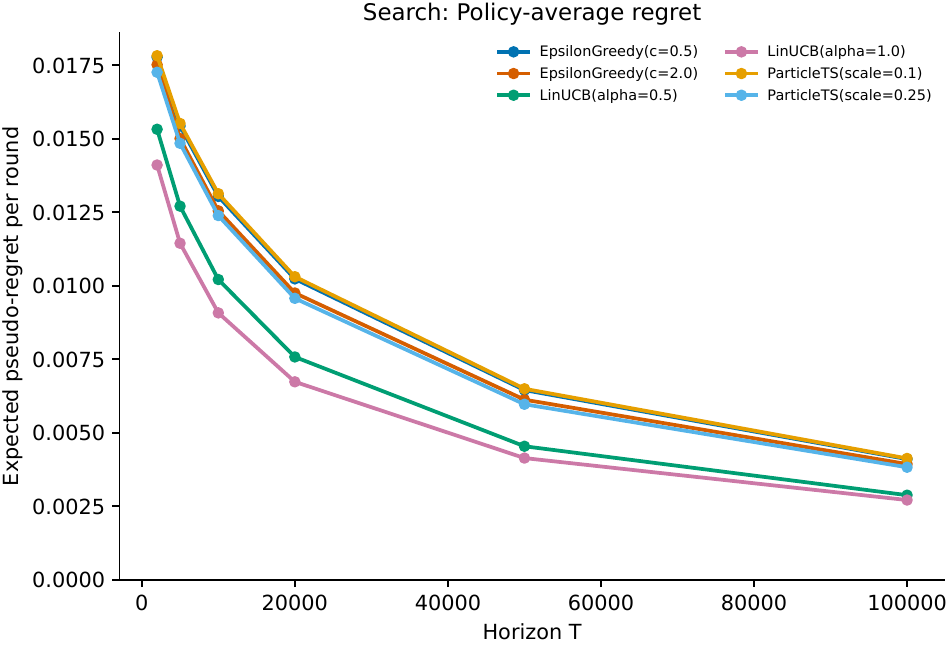}
  \caption{policy regret.}
  \label{fig:search-regret}
\end{subfigure}

\caption{
  Query cost, inferential sensitivity, and estimation accuracy for MSLR-Search. (d) plots the average pseudo-regret of each policy.
}
\label{fig:fixed-policy-cost-sensitivity-search}
\end{figure*}

Consistent with Theorem \ref{thm:regret-cost}, the normalized query cost decreases as the adaptive policies' per-round regrets fall, from roughly 80\% toward 50\% of the full $JT$ budget over the reported horizon.  Plot (b) shows a widening MDE80 advantage at matched cost.  For a realized tree cost $N(T)$, the matched-budget independent baseline can advance each of the $J$ adaptive policies for only $q_T=\lfloor N(T)/J\rfloor$ rounds.  To display an estimate against the horizon-$T$ cumulative target, the experiment rescales it as $T S_j(q_T)/q_T$.  This extrapolation is generally biased during learning because the average reward over the first $q_T$ rounds need not equal the average reward over the first $T$ rounds; the same issue does not arise for the exact horizon-$T$ \TCAB\ trajectory.

\bibliographystyle{abbrvnat}
\bibliography{Preprint}

\clearpage
\appendix

\section{Coupling preliminaries}
\label{app:coupling}

\subsection{Two-marginal maximal coupling}

Let $(\mathsf Z,\mathcal Z)$ be standard Borel and let $\nu_0,\nu_1$ be dominated by $\rho$, with densities $p_0,p_1$.  Define
\[
\lambda(\dd z):=\min\{p_0(z),p_1(z)\}\rho(\dd z),
\qquad
\alpha:=\lambda(\mathsf Z),
\qquad
\delta:=1-\alpha.
\]
When $\delta>0$, define
\[
\widetilde\nu_i(\dd z)
:=\frac{\nu_i(\dd z)-\lambda(\dd z)}{\delta},
\qquad i\in\{0,1\}.
\]

\begin{lemma}[Canonical maximal coupling]
\label{lem:canonical-maximal}
With probability $\alpha$, draw $Z$ from $\lambda/\alpha$ and set $Z_0=Z_1=Z$; this branch is absent when $\alpha=0$.  With probability $\delta$, draw $(Z_0,Z_1)$ from any coupling of $(\widetilde\nu_0,\widetilde\nu_1)$; this branch is absent when $\delta=0$.  Then $Z_i\sim\nu_i$ and
\[
\Pp(Z_0\ne Z_1)=\delta=\TV(\nu_0,\nu_1).
\]
No coupling has a smaller disagreement probability.
\end{lemma}

\begin{proof}
The marginal decomposition $\nu_i=\lambda+\delta\widetilde\nu_i$ gives the required laws.  The residual measures are mutually singular because their densities are supported on $\{p_0>p_1\}$ and $\{p_1>p_0\}$, respectively, so the residual branch cannot agree with positive probability.  Also
\[
\TV(\nu_0,\nu_1)
=1-\int\min\{p_0,p_1\}\dd\rho
=\delta.
\]
For any coupling, the common-value submeasure is dominated by both marginals and therefore has mass at most $\int\min\{p_0,p_1\}\dd\rho=\alpha$.  Hence $\Pp(Z_0\ne Z_1)\ge\delta$.  See \citet{thorisson2000coupling} and \citet{lindvall2002lectures}.
\end{proof}

\begin{lemma}[One-sided rejection sampler]
\label{lem:one-sided-rejection}
Suppose samples from $\nu_0,\nu_1$ and the density ratios are available.  Draw $Z_0\sim\nu_0$.  Set $Z_1=Z_0$ with probability $\min\{1,p_1(Z_0)/p_0(Z_0)\}$.  Otherwise repeatedly draw $Y\sim\nu_1$ until it is accepted with probability
\[
1-\min\{1,p_0(Y)/p_1(Y)\},
\]
and set $Z_1=Y$.  This is a maximal coupling.  The first-stage match probability is $\alpha$.  If $\delta>0$, conditional on entering the residual branch, the number of $\nu_1$ proposals is geometric with success probability $\delta$ and mean $1/\delta$; when $\delta=0$, the branch is never entered.
\end{lemma}

\begin{proof}
The accepted first-stage measure is
\[
p_0(z)\min\{1,p_1(z)/p_0(z)\}\rho(\dd z)
=\lambda(\dd z).
\]
A residual proposal is accepted with unnormalized density
\[
p_1(z)\left[1-\min\{1,p_0(z)/p_1(z)\}\right]
=(p_1(z)-p_0(z))_+,
\]
whose total mass is $\delta$.  Thus the accepted residual has law $\widetilde\nu_1$ and the proposal count has the stated geometric law.
\end{proof}

\subsection{Tree gluing}

\begin{lemma}[Tree gluing lemma]
\label{lem:tree-gluing}
Let $\Tcal=(\Ical,E)$ be a finite rooted tree on a standard Borel space $\mathsf Z$.  For every directed edge $e=(p(v),v)$, let $\gamma_e$ be a coupling of $\nu_{p(v)}$ and $\nu_v$.  Then there is a joint coupling of $(\nu_j)_{j\in\Ical}$ whose marginal on every edge is $\gamma_e$.  In particular, maximal couplings can be attained simultaneously on all tree edges.
\end{lemma}

\begin{proof}
Disintegrate
\[
\gamma_e(\dd z_p,\dd z_v)
=\nu_p(\dd z_p)K_e(z_p,\dd z_v),
\]
which is possible on standard Borel spaces.  Draw $Z_r\sim\nu_r$ and then recursively draw
\[
Z_v\sim K_{(p(v),v)}(Z_{p(v)},\cdot),
\]
conditionally independently across siblings if desired.  Induction over depth gives every node marginal $\nu_v$, and the construction gives the prescribed edge law.  Acyclicity removes compatibility constraints that can arise on graphs with cycles.
\end{proof}

\section{Proof of the implementation proposition}
\label{app:implementation-proof}

\begin{proof}[Proof of Proposition \ref{prop:implementation}]
Fix a round, an oriented edge $(p,v)$, and the two current histories.  Apply Lemma \ref{lem:one-sided-rejection} to
\[
\nu_p(\dd x,a)=P_X(\dd x)\pi_{p,t}(a\mid h_p,x),
\qquad
\nu_v(\dd x,a)=P_X(\dd x)\pi_{v,t}(a\mid h_v,x),
\]
with dominating measure $P_X\otimes\mathrm{counting}$.  The common $P_X$ factor cancels, yielding exactly \eqref{eq:edge-first-accept} and \eqref{eq:edge-residual-accept}.  Lemma \ref{lem:tree-gluing} glues these edge laws into the round-$t$ tree coupling used in Algorithm \ref{alg:tcab}.  The residual branch is entered with probability $\delta_{pv,t}$ and, when positive, its conditional proposal mean is $1/\delta_{pv,t}$; its unconditional contribution is therefore one.  Summing over at most $J-1$ edges in each of $T$ rounds proves the proposal bound.  Reward sampling is performed only after the complete-pair coupling and therefore does not enter the acceptance ratios.
\end{proof}

\section{Proofs of exactness and query cost}
\label{app:exactness}

\begin{proof}[Proof of Theorem \ref{thm:marginals}]
Fix $t$ and condition on $\Fcal_{t-1}$.  The tree $\Tcal_t$ and every policy history are then fixed.  By the canonical edge construction and Lemma \ref{lem:tree-gluing},
\begin{equation}
\mathcal L(Z_t^j\mid\Fcal_{t-1})
=\nu_{j,t}^{H_{t-1}^j}
\label{eq:conditional-node-law-proof}
\end{equation}
for every $j$.  In particular, for measurable $B\subseteq\X$ and $a\in\A$,
\[
\Pp(X_t^j\in B,A_t^j=a\mid\Fcal_{t-1})
=\int_B\pi_{j,t}(a\mid H_{t-1}^j,x)P_X(\dd x).
\]

Now condition further on all round-$t$ complete pairs and their matched components.  A component with common pair $(x,a)$ receives a fresh draw from $Q_a(\cdot\mid x)$, independent of the coupling randomness.  Thus, for every node,
\begin{equation}
\mathcal L(R_t^j\mid\Fcal_{t-1},Z_t^j=(x,a))
=Q_a(\cdot\mid x).
\label{eq:conditional-reward-law-proof}
\end{equation}
Sharing the draw across a component changes cross-policy dependence but not any node's conditional law.  Equations \eqref{eq:conditional-node-law-proof}--\eqref{eq:conditional-reward-law-proof} are exactly the standalone transition \eqref{eq:standalone-transition}.  Starting from empty histories and iterating proves equality of each complete marginal path law by induction, equivalently by uniqueness in the Ionescu--Tulcea theorem.
\end{proof}

\begin{proof}[Proof of Corollary \ref{cor:unbiased}]
Theorem \ref{thm:marginals} gives the standalone distribution, hence expectation, of every $S_j(T)$.  Linearity gives the baseline-contrast identities.
\end{proof}

\begin{proof}[Proof of Theorem \ref{thm:tree-cost}]
Removing the mismatched edges from a tree creates exactly one more component per removed edge.  Because \TCAB\ queries one reward per component, \eqref{eq:pathwise-tree-cost} holds.  Conditional on $\Fcal_{t-1}$, every selected edge is a maximal coupling, so
\[
\E[D_{e,t}\mid\Fcal_{t-1}]=\delta_{e,t}.
\]
Taking conditional expectation in the pathwise identity, then summing and taking ordinary expectations, proves \eqref{eq:tree-cost-identity}.

For the lower bound, fix $\Fcal_{t-1}$ and consider any conditionally exact edge-local design on the same selected tree.  On edge $e=\{i,j\}$, conditional exactness fixes the two node marginals as $\nu_{i,t}^{H_{t-1}^i}$ and $\nu_{j,t}^{H_{t-1}^j}$.  The coupling inequality therefore gives
\[
\Pp(Z_t^i\ne Z_t^j\mid\Fcal_{t-1})\ge\delta_{e,t}.
\]
Edge locality forces a new reward lineage whenever a child mismatches its parent, so its conditional expected cost is at least
\[
1+\sum_{e\in E_t}
\Pp(Z_t^i\ne Z_t^j\mid\Fcal_{t-1})
\ge1+\sum_{e\in E_t}\delta_{e,t}.
\]
Tree gluing shows that all edgewise lower bounds can coexist, and \TCAB\ attains them.
\end{proof}

\begin{proof}[Proof of Corollary \ref{cor:star-cost}]
Apply Theorem \ref{thm:tree-cost} to the edges $\{r,j\}$, $j\ne r$.  The lower-bound class specializes to designs in which every alternative can inherit only from the baseline.
\end{proof}

\begin{proof}[Proof of Corollary \ref{cor:mst-tree}]
For any selected tree, Theorem \ref{thm:tree-cost} gives minimal conditional tree-local cost $1+\sum_{e\in E}\delta_{e,t}$.  Minimizing this sum over spanning trees is exactly the minimum-spanning-tree problem.  There are finitely many labeled trees; a deterministic lexicographic rule among minimizers makes the choice measurable in the edge weights and hence predictable.  Applying the tree coupling preserves every node marginal.  The argument is pointwise in the current histories and therefore establishes only the stated current-round result.
\end{proof}

\section{Proofs of regret-based efficiency}
\label{app:regret}

\begin{proof}[Proof of Theorem \ref{thm:regret-cost}]
For history $h_j$, define the conditional nonoracle probability
\[
q_{j,t}(h_j)
:=\int_{\X}
\left[1-\pi_{j,t}(a^*(x)\mid h_j,x)\right]P_X(\dd x).
\]
For an edge $\{i,j\}$ and fixed context $x$, write $p_i,p_j$ for the two current action distributions.  Their overlap includes at least the smaller mass assigned to the common oracle action, so
\begin{align*}
\TV(p_i,p_j)
&=1-\sum_a\min\{p_i(a),p_j(a)\}\\
&\le1-\min\{p_i(a^*(x)),p_j(a^*(x))\}\\
&\le[1-p_i(a^*(x))]+[1-p_j(a^*(x))].
\end{align*}
Integrating over $P_X$ gives
\[
\delta_{ij,t}\le q_{i,t}(H_{t-1}^i)+q_{j,t}(H_{t-1}^j).
\]
Summing over the round-$t$ tree and then over time yields
\begin{equation}
\E[N(T)]-T
\le
\E\!\left[
\sum_{t=1}^T\sum_{j\in\Ical}
 d_{j,t}\ind\{A_t^j\ne a^*(X_t^j)\}
\right].
\label{eq:degree-mistake-bound-proof}
\end{equation}
Here we used conditional exactness of \TCAB\ to identify $q_{j,t}$ with the conditional nonoracle probability of policy $j$.

For every $\eps>0$,
\[
d_{j,t}\ind\{A_t^j\ne a^*(X_t^j)\}
\le
 d_{j,t}\ind\{\Delta(X_t^j)\le\eps\}
+\frac{d_{j,t}}{\eps}\Delta_{A_t^j}(X_t^j).
\]
Because $d_{j,t}$ is $\Fcal_{t-1}$-measurable and $X_t^j$ has conditional marginal $P_X$,
\[
\sum_{j\in\Ical}
\E\!\left[d_{j,t}\ind\{\Delta(X_t^j)\le\eps\}\right]
=F_\Delta(\eps)\E\!\left[\sum_jd_{j,t}\right]
=2(J-1)F_\Delta(\eps).
\]
Summing over $t$ and using \eqref{eq:degree-regret} in \eqref{eq:degree-mistake-bound-proof} proves \eqref{eq:finite-regret-cost}.

For fixed $J$, \eqref{eq:degree-regret-upper} and $\mathcal R_j(T)=o(T)$ imply $r_T:=\mathcal R_{\deg}(T)/T\to0$.  If $r_T>0$, choose $\eps_T=\sqrt{r_T}$; otherwise choose any deterministic $\eps_T\downarrow0$.  Almost-sure uniqueness and finite $K$ imply $\Delta(X)>0$ almost surely and hence $F_\Delta(\eps_T)\to0$.  Divide \eqref{eq:finite-regret-cost} by $T$ to obtain \eqref{eq:regret-near-one}.
\end{proof}

\begin{proof}[Proof of Corollary \ref{cor:margin-rate}]
Under the margin condition, \eqref{eq:finite-regret-cost} is at most
\[
2C(J-1)T\eps^\beta+\frac{\mathcal R_{\deg}(T)}{\eps}.
\]
Balancing the two terms gives
\[
\eps\asymp
\left(
\frac{\mathcal R_{\deg}(T)}{(J-1)T}
\right)^{1/(\beta+1)},
\]
which proves \eqref{eq:margin-rate}.  Under a uniform gap, every nonoracle action incurs loss at least $\Delta_{\min}$, so the right side of \eqref{eq:degree-mistake-bound-proof} is at most $\mathcal R_{\deg}(T)/\Delta_{\min}$.
\end{proof}

\section{Proofs of the variance results}
\label{app:variance}

\begin{proof}[Proof of Theorem \ref{thm:pairwise-variance}]
Fix $i,j$ and define
\[
U_t^{ij}:=
\ind\{\text{at least one edge in }P_{ij,t}\text{ mismatches}\}.
\]
Then
\begin{equation}
U_t^{ij}\le\sum_{e\in P_{ij,t}}D_{e,t}.
\label{eq:path-union-proof}
\end{equation}
If $U_t^{ij}=0$, equality propagates along the matched path, so $Z_t^i=Z_t^j$ and $R_t^i=R_t^j$.

Define
\[
C_t^{ij}:=\mu_*(X_t^j)-\mu_*(X_t^i).
\]
Both endpoint contexts have conditional marginal $P_X$, hence
$\E[C_t^{ij}\mid\Fcal_{t-1}]=0$.  Also $C_t^{ij}=0$ when $U_t^{ij}=0$ and $|C_t^{ij}|\le1$.  Therefore
\begin{equation}
\Var\!\left(\sum_{t=1}^TC_t^{ij}\right)
\le\E\!\left[\sum_{t=1}^TU_t^{ij}\right].
\label{eq:context-martingale-bound-proof}
\end{equation}

Let $\Gcal_t$ enlarge $\Fcal_{t-1}$ by all round-$t$ complete pairs, the tree, and the component partition, but not the reward draws.  Define
\[
E_t^{ij}
:=\{R_t^j-\mu_{A_t^j}(X_t^j)\}
 -\{R_t^i-\mu_{A_t^i}(X_t^i)\}.
\]
Conditional on $\Gcal_t$, each residual has mean zero.  Thus $\E[E_t^{ij}\mid\Gcal_t]=0$.  Moreover, $E_t^{ij}=0$ when $U_t^{ij}=0$ and $|E_t^{ij}|\le2U_t^{ij}$, giving
\begin{equation}
\Var\!\left(\sum_{t=1}^TE_t^{ij}\right)
\le4\E\!\left[\sum_{t=1}^TU_t^{ij}\right].
\label{eq:reward-martingale-bound-proof}
\end{equation}

Writing $\ell_t^j:=\Delta_{A_t^j}(X_t^j)$,
\[
R_t^j-R_t^i
=C_t^{ij}+\ell_t^i-\ell_t^j+E_t^{ij}.
\]
After summing over $t$, use
\[
\Var(U+V+W)\le3\{\Var(U)+\Var(V)+\Var(W)\},
\qquad
\Var(L_i-L_j)\le2\Var(L_i)+2\Var(L_j),
\]
together with \eqref{eq:path-union-proof}--\eqref{eq:reward-martingale-bound-proof}.  This yields
\[
\Var(S_j-S_i)
\le15\sum_{t=1}^T\E\!\left[\sum_{e\in P_{ij,t}}D_{e,t}\right]
+6\Var(L_i)+6\Var(L_j),
\]
which is \eqref{eq:path-variance-bound}.
\end{proof}

\subsection{General fixed zero-sum contrasts}
\label{app:zero-sum}

The pairwise result immediately extends to comparative linear functionals.  Let
\[
S(T):=(S_0(T),\ldots,S_{J-1}(T))^\top,
\qquad c\in\R^J,
\qquad \mathbf 1^\top c=0,
\]
and define the corresponding target by
\[
\theta_c(T):=c^\top V(T).
\]
By Theorem \ref{thm:marginals}, $c^\top S(T)$ is unbiased for $\theta_c(T)$ whenever rewards have finite first moments.

\begin{theorem}[Fixed zero-sum contrasts]
\label{thm:contrast-variance}
Fix an anchor $r\in\Ical$.  Then
\begin{equation}
\Var\!\left(c^\top S(T)\right)
\le
\left(
\sum_{j\ne r}|c_j|\sqrt{B_{jr}(T)}
\right)^2.
\label{eq:general-contrast-variance}
\end{equation}
If $J,c$ are fixed, $\E[N(T)]-T=o(T)$, and $\Var(L_j(T))=o(T)$ for every policy, then $\Var(c^\top S(T))=o(T)$.
\end{theorem}

\begin{corollary}[Simple second-moment condition]
\label{cor:mistake-second-moment}
Let
\[
M_j(T):=\sum_{t=1}^T\ind\{A_t^j\ne a^*(X_t^j)\}.
\]
If $\E[N(T)]-T=o(T)$ and $\E[M_j(T)^2]=o(T)$ for every $j$, then every fixed zero-sum contrast has variance $o(T)$.
\end{corollary}

\begin{proof}[Proof of Theorem \ref{thm:contrast-variance}]
Because $\sum_jc_j=0$,
\[
c^\top S(T)=\sum_{j\ne r}c_j\{S_j(T)-S_r(T)\}.
\]
Center both sides and apply Minkowski's inequality in $L^2$:
\[
\sqrt{\Var(c^\top S(T))}
\le\sum_{j\ne r}|c_j|
\sqrt{\Var(S_j(T)-S_r(T))}.
\]
Theorem \ref{thm:pairwise-variance} gives \eqref{eq:general-contrast-variance}.  Moreover,
\[
\sum_{t=1}^T\E\!\left[\sum_{e\in P_{jr,t}}D_{e,t}\right]
\le\E[N(T)]-T.
\]
Thus every $B_{jr}(T)=o(T)$ under the stated assumptions, and a fixed finite sum of $o(\sqrt T)$ terms has square $o(T)$.
\end{proof}

\begin{proof}[Proof of Corollary \ref{cor:mistake-second-moment}]
Since rewards lie in $[0,1]$, every gap lies in $[0,1]$ and
\[
0\le L_j(T)\le M_j(T).
\]
Hence
\[
\Var(L_j(T))\le\E[L_j(T)^2]
\le\E[M_j(T)^2]=o(T).
\]
Apply Theorem \ref{thm:contrast-variance}.
\end{proof}

\end{document}

%% file: defs.tex
\usepackage{bbm}
\usepackage{graphicx}
\usepackage{amsmath,amssymb,amsthm,amsfonts}

\usepackage{paralist}
\usepackage{bm}
\usepackage{xspace}
\usepackage{url}
\usepackage{prettyref}
\usepackage{boxedminipage}
\usepackage{wrapfig}
\usepackage{ifthen}
\usepackage{color}
\usepackage{xspace}

\usepackage{amsmath,amsthm,amsfonts,amssymb}
\usepackage{mathtools}
\usepackage{graphicx}

\usepackage{nicefrac}

\newtheorem*{definition*}{Definition}

\DeclareMathOperator*{\argmax}{arg\,max}

\usepackage{subcaption}

\usepackage[utf8]{inputenc}

\usepackage{xcolor}
\definecolor{expert}{HTML}{008000}
\definecolor{error}{HTML}{f96565}

\newcommand{\norm}[1]{\left\lVert #1 \right\rVert}

\usepackage{color-edits}
\addauthor{sw}{blue}

\usepackage{thmtools}
\usepackage{thm-restate}

\usepackage{tikz}
\usetikzlibrary{arrows,calc} 
\newcommand{\tikzAngleOfLine}{\tikz@AngleOfLine}
\def\tikz@AngleOfLine(#1)(#2)#3{%
\pgfmathanglebetweenpoints{%
\pgfpointanchor{#1}{center}}{%
\pgfpointanchor{#2}{center}}
\pgfmathsetmacro{#3}{\pgfmathresult}%
}

\declaretheoremstyle[
    headfont=\normalfont\bfseries, 
    bodyfont = \normalfont\itshape]{mystyle}

\usepackage[linesnumbered,algoruled,boxed,lined,noend]{algorithm2e}

\usepackage{listings}
\usepackage{amsmath}
\usepackage{amsthm}
\usepackage{tikz}
\usepackage{caption}
\usepackage{mdwmath}
\usepackage{multirow}
\usepackage{mdwtab}
\usepackage{eqparbox}
\usepackage{multicol}
\usepackage{amsfonts}
\usepackage{tikz}
\usepackage{multirow,bigstrut,threeparttable}
\usepackage{amsthm}
\usepackage{bbm}
\usepackage{epstopdf}
\usepackage{mdwmath}
\usepackage{mdwtab}
\usepackage{eqparbox}
\usetikzlibrary{topaths,calc}
\usepackage{latexsym}
\usepackage{cite}
\usepackage{amssymb}
\usepackage{bm}
\usepackage{amssymb}
\usepackage{graphicx}
\usepackage{mathrsfs}
\usepackage{epsfig}
\usepackage{psfrag}
\usepackage{setspace}
\usepackage[
            CJKbookmarks=true,
            bookmarksnumbered=true,
            bookmarksopen=true,
            colorlinks=true,
            citecolor=red,
            linkcolor=blue,
            anchorcolor=red,
            urlcolor=blue
            ]{hyperref}
\usepackage[linesnumbered,algoruled,boxed,lined]{algorithm2e}
\usepackage{algpseudocode}
\usepackage{stfloats}
\RequirePackage[authoryear]{natbib}

\usepackage{comment}
\usepackage{mathtools}
\usepackage{blkarray}
\usepackage{multirow,bigdelim,dcolumn,booktabs}

\usepackage{xparse}
\usepackage{tikz}
\usetikzlibrary{calc}
\usetikzlibrary{decorations.pathreplacing,matrix,positioning}

\usepackage[T1]{fontenc}
\usepackage[utf8]{inputenc}
\usepackage{mathtools}
\usepackage{blkarray, bigstrut}
\usepackage{gauss}

\newcommand*{\BraceAmplitude}{0.4em}%
\newcommand*{\VerticalOffset}{0.5ex}%
\newcommand*{\HorizontalOffset}{0.0em}%
\newcommand*{\blocktextwid}{3.0cm}%
\NewDocumentCommand{\InsertLeftBrace}{%
	O{} 
	O{\HorizontalOffset,\VerticalOffset} 
	O{\blocktextwid} 
	m   
	m   
	m   
}{%
	\begin{tikzpicture}[overlay,remember picture]
	\coordinate (Brace Top)    at ($(#4.north) + (#2)$);
	\coordinate (Brace Bottom) at ($(#5.south) + (#2)$);
	\draw [decoration={brace, amplitude=\BraceAmplitude}, decorate, thick, draw=black, #1]
	(Brace Bottom) -- (Brace Top) 
	node [pos=0.5, anchor=east, align=left, text width=#3, color=black, xshift=\BraceAmplitude] {#6};
	\end{tikzpicture}%
}%
\NewDocumentCommand{\InsertRightBrace}{%
	O{} 
	O{\HorizontalOffset,\VerticalOffset} 
	O{\blocktextwid} 
	m   
	m   
	m   
}{%
	\begin{tikzpicture}[overlay,remember picture]
	\coordinate (Brace Top)    at ($(#4.north) + (#2)$);
	\coordinate (Brace Bottom) at ($(#5.south) + (#2)$);
	\draw [decoration={brace, amplitude=\BraceAmplitude}, decorate, thick, draw=black, #1]
	(Brace Top) -- (Brace Bottom) 
	node [pos=0.5, anchor=west, align=left, text width=#3, color=black, xshift=\BraceAmplitude] {#6};
	\end{tikzpicture}%
}%
\NewDocumentCommand{\InsertTopBrace}{%
	O{} 
	O{\HorizontalOffset,\VerticalOffset} 
	O{\blocktextwid} 
	m   
	m   
	m   
}{%
	\begin{tikzpicture}[overlay,remember picture]
	\coordinate (Brace Top)    at ($(#4.west) + (#2)$);
	\coordinate (Brace Bottom) at ($(#5.east) + (#2)$);
	\draw [decoration={brace, amplitude=\BraceAmplitude}, decorate, thick, draw=black, #1]
	(Brace Top) -- (Brace Bottom) 
	node [pos=0.5, anchor=south, align=left, text width=#3, color=black, xshift=\BraceAmplitude] {#6};
	\end{tikzpicture}%
}%

\usetikzlibrary{patterns}

\definecolor{cof}{RGB}{219,144,71}
\definecolor{pur}{RGB}{186,146,162}
\definecolor{greeo}{RGB}{91,173,69}
\definecolor{greet}{RGB}{52,111,72}

\theoremstyle{plain}
\newtheorem{theorem}{Theorem}

\newtheorem{lemma}{Lemma}

\newtheorem{corollary}{Corollary}
\newtheorem{definition}{Definition}

\newtheorem{assumption}{Assumption}

\def\1{\mathbbm{1}}

\newenvironment{keywords}
{\bgroup\leftskip 20pt\rightskip 20pt \small\noindent{\bfseries
Keywords:} \ignorespaces}%
{\par\egroup\vskip 0.25ex}
\newlength\aftertitskip     \newlength\beforetitskip
\newlength\interauthorskip  \newlength\aftermaketitskip

\usepackage{xspace}

\newcommand{\TV}{{\sf TV}}

\definecolor{myblue}{rgb}{.8, .8, 1}
\definecolor{mathblue}{rgb}{0.2472, 0.24, 0.6} 
\definecolor{mathred}{rgb}{0.6, 0.24, 0.442893}
\definecolor{mathyellow}{rgb}{0.6, 0.547014, 0.24}

\usepackage{cleveref}
\crefname{lemma}{Lemma}{Lemmas}
\Crefname{lemma}{Lemma}{Lemmas}
\crefname{thm}{Theorem}{Theorems}
\Crefname{thm}{Theorem}{Theorems}
\Crefname{assumption}{Assumption}{Assumptions}
\Crefname{inftheorem}{Informal Theorem}{Informal Theorems}
\crefformat{equation}{(#2#1#3)}